\def\year{2020}\relax
\documentclass[letterpaper]{article}
\usepackage{aaai20}
\usepackage{times}
\usepackage{helvet}
\usepackage{courier}
\usepackage{url}
\usepackage{graphicx}
\usepackage{graphicx}
\usepackage{subfigure}
\usepackage{amsfonts}       
\usepackage{nicefrac}       
\usepackage{amsmath}
\usepackage{amsthm}
\usepackage{algorithm}
\usepackage{algorithmic}
\newtheorem{lemm}{Lemma}
\newtheorem{proposition}{Proposition}

\title{Adaptive Double-Exploration Tradeoff for Outlier Detection}
\author{
  Xiaojin Zhang,\textsuperscript{\rm 1}
  Honglei Zhuang,\textsuperscript{\rm 2}
  Shengyu Zhang,\textsuperscript{\rm 3}
  Yuan Zhou\textsuperscript{\rm 2}\\
  \textsuperscript{\rm 1}The Chinese University of Hong Kong,
  \textsuperscript{\rm 2}University of Illinois Urbana-Champaign,
  \textsuperscript{\rm 3}Tencent\\
  xjzhang@cse.cuhk.edu.hk,
  hzhuang3@illinois.edu,
  shengyuzhang@gmail.com,
  yuanz@illinois.edu}

\begin{document}
\maketitle
\begin{abstract}
We study a variant of the thresholding bandit problem (TBP) in the context of outlier detection, where the objective is to identify the outliers whose rewards are above a threshold. Distinct from the traditional TBP, the threshold is defined as a function of the rewards of all the arms, which is motivated by the criterion for identifying outliers. The learner needs to explore the rewards of the arms as well as the threshold. We refer to this problem as "double exploration for outlier detection". We construct an adaptively updated confidence interval for the threshold, based on the estimated value of the threshold in the previous rounds. Furthermore, by automatically trading off exploring the individual arms and exploring the outlier threshold, we provide an efficient algorithm in terms of the sample complexity. Experimental results on both synthetic datasets and real-world datasets demonstrate the efficiency of our algorithm.
\end{abstract}

\section{Introduction}


Multi-armed bandit (MAB) problems model the tradeoff between \textit{exploration} and \textit{exploitation} inherent in a great amount of sequential decision problems. In the canonical multi-armed bandit problem, the learner is presented with a set of arms. The learner needs to explore distinct arms to discover the potential of various arms, and exploit the most rewarding arms based on the information that has been collected. The goal of the learner is to maximize the cumulative reward \cite{robbins1985some,auer2002finite}. In terms of the \textit{pure exploration} problem, the learner is interested in optimizing the performance in some decision-making tasks by making full use of the limited budget \cite{bubeck2011pure,gabillon2011multi}. A flurry of work focuses on identifying the set of $K$ arms with the largest expected rewards, called the top-$k$ arm identification problem \cite{mnih2008empirical,kalyanakrishnan2010efficient,kalyanakrishnan2012pac,kaufmann2013information,zhou2014optimal,chen2014combinatorial}. This line of work could be categorized into two distinct settings: \textit{fixed confidence} and \textit{fixed budget}. The former setting focuses on identifying the best action satisfying a fixed confidence level using the minimum number of samples, while the latter one aims at maximizing the probability of outputting the best action using a fixed number of samples.

In this paper, we focus on the outlier detection problem, a practically important problem which was firstly investigated in the framework of multi-armed bandit under the \textit{fixed confidence} setting by \citeauthor{zhuang2017identifying} \shortcite{zhuang2017identifying}. They regard the outliers as the arms with rewards above a certain threshold, which is defined based on a generalized statistical technique called $k$-sigma rule of thumb \cite{coolidge2012statistics}. Specifically, an arm is referred to as an outlier, if and only if the reward of which lies above $k$ standard deviations of the mean. The expected reward of each arm is unknown to the learner. Moreover, the threshold is a function of the rewards of all the arms and is also unknown, which is the distinguishing feature from the classical thresholding bandit problem and top-$k$ arm identification problem. The learner could select an arm to pull at each step, and the reward drawn from the distribution of pulled arm is observed as the feedback information. In this process, the rewards of the arms and the threshold are gradually learned. The final goal of the learner is to output the correct outlier set with high probability, and at the same time minimize the number of samples used.

The outlier detection problem in the framework of multi-armed bandit raises two main challenges to the design of algorithms. On the one hand, the learning strategies for outlier detection need to address the \textit{double-exploration} dilemma, i.e. the search for a balance between exploring the arms and exploring the threshold. If the learner plays exclusively on the arms that are closer to the threshold, he might fail to estimate the threshold efficiently. If the learner is persistent in sampling all the arms uniformly, he might spend too much effort on the arms that have already been identified. On the other hand, the construction of the confidence interval for the threshold, which is related to the standard deviation of the rewards of all the arms, is not an easy task. The tighter the confidence interval is, the smaller sample complexity the resulting algorithm could obtain.

We note that the confidence radius for the threshold constructed by \citeauthor{zhuang2017identifying} \shortcite{zhuang2017identifying} increases with the number of arms, which might make the learning strategy inefficient in terms of the sample complexity. This issue becomes more severe when identifying outliers among a larger set of arms. Our approach alleviates this problem by constructing a confidence radius for the threshold which is independent of the number of arms and could be tuned adaptively. Besides, their sampling strategies are forced to continue sampling the arms that have already been identified, since the confidence radius for the threshold is related to the harmonic mean of the number of samples of all the arms. We devise distinct sampling strategies for estimating the expected rewards of the arms and the thresholds separately, and make it feasible to directly remove the arms that have been identified. We provide an algorithm that balances between these two sampling strategies in an adaptive manner, and provide a theoretical guarantee in terms of both correctness and sampling complexity. We further apply our algorithm to both synthetic datasets and real-world datasets. Experimental results demonstrate that our algorithm achieves considerable improvement over the state-of-the-art algorithms for outlier detection in the MAB framework.








\section{Related Work}








Since the first introduction of the multi-armed bandit by \citeauthor{thompson1933likelihood} \shortcite{thompson1933likelihood} in the scenario of medical trials, it has received a great amount of interest. The goal of the learner is to minimize the cumulative regret. In the pure exploration setting, the learner is assessed in terms of the simple regret instead of the cumulative regret. \citeauthor{locatelli2016optimal} \shortcite{locatelli2016optimal} investigates a specific pure exploration problem in the fixed budget setting, referred to as the thresholding bandit problem. This problem aims at finding the arms with rewards larger than a given threshold within a fixed time horizon.

There has been a large body of work on the problem of best arm identification in both fixed confidence and fixed budget setting in the literature \cite{even2002pac,bubeck2011pure,gabillon2012best,jamieson2014lil,kaufmann2016complexity}. A line of work uses the elimination-based approach, which successively removes the arms that are believed to be suboptimal with certain confidence \cite{paulson1964sequential,maron1994hoeffding,even2002pac}. \citeauthor{mnih2008empirical} \shortcite{mnih2008empirical} proposed an algorithm taking advantage of the empirical variance of the arms. This problem is further extended to the top-$k$ arm identification problem.



Our problem fits into the fixed confidence framework, but is distinct from the top-$k$ arm identification problem. The number of outliers is not fixed, and depends on the rewards of all the arms, thereby could not be directly reduced to the top-$k$ arm identification problem. Another line of work relaxed the classical optimal arm identification problem to $\epsilon$-optimal arm identification, aiming at finding an arm which is at least $\epsilon$-close to the optimal arm with probability at least $1-\delta$  \cite{domingo2002adaptive,even2006action}. 
\citeauthor{kalyanakrishnan2010efficient} \shortcite{kalyanakrishnan2010efficient} generalized the setting as identifying $(\epsilon, m)$-optimal arm, the rewards of which is within $\epsilon$ of the $m$-th optimal arm. \citeauthor{kalyanakrishnan2012pac} \shortcite{kalyanakrishnan2012pac} further proposed the LUCB algorithm based on the upper and lower confidence bound for the generalized case. In this paper, we focus on the case when $\epsilon$ is $0$. That is, the arms with rewards above the exact threshold are regarded as the outliers.






The problem of outlier detection has been widely investigated in the field of data mining \cite{chandola2009anomaly}. The approaches used for outlier detection include classification-based \cite{de2000reject,roth2006kernel}, clustering-based \cite{ester1996density,yu2002findout}, nearest neighbor-based \cite{guttormsson1999elliptical,kou2006spatial}, and statistical techniques \cite{torr1993outlier,horn2001effect}. Most existing approaches on outlier detection do not belong to the domain of online algorithms. The first work casting the outlier detection problem into the MAB framework belongs to a recent work proposed by  \citeauthor{zhuang2017identifying} \shortcite{zhuang2017identifying}, which is the work mostly related to ours. They proposed two algorithms named \textup{RR} and \textup{WRR}, aiming at finding the correct outlier set with high probability. The rewards of the arms and the threshold are unknown but could be learned by sampling. RR algorithm simply samples each arm evenly in an iterative manner, while WRR algorithm allocates more samples to the arms that are not yet identified. The confidence radius for the threshold constructed in their work scales with the number of arms, which may greatly hinder the proposal of an efficient algorithm. Our work overcomes this deficiency by constructing a confidence radius for the threshold which is independent of the number of arms (if we ignore logarithmic factors), and provides a more flexible trade-off between exploring the arms and exploring the threshold.

\section{Problem Formulation}






In this section, we introduce the outlier detection problem in the framework of multi-armed bandits. The learner is presented with a set of $n$ arms which are enumerated by $[n] = \{1,2,\dots,n\}$. The reward distribution associated with each arm $i$ is bounded in $[a,b]$ with mean $y_i$. Without loss of generality, we assume that $a\ge 0$. Let $R = b-a$, and $R' = b^2-a^2$. Each pull of arm $i\in [n]$ generates a sample drawn from the distribution corresponding to arm $i$, which is independent of the historical pulls. The term "outlier" refers to the arm whose reward lies above $k$ standard deviations of the mean. Consequently, the threshold for distinguishing the outliers from the normal arms is defined as:
\begin{align} 
\theta = \mu_y + k\sigma_y,
\end{align}
where $\mu_y =\displaystyle\frac{1}{n} \sum_{i=1}^n y_i$ is the mean of all the arms, $\sigma_y = \sqrt{\displaystyle\frac{1}{n}\sum_{i=1}^n(y_i - \mu_y)^2}$ is the standard deviation of all the arms, and $k$ is a constant that could be set based on specific application scenario.





Let $\mathcal O = \{S: S\subset [n] \}$ be the set of all subsets of $[n]$. Define $O^* = \{i\in [n]: y_i\ge\theta\}$ as the correct outlier set which contains the arms with expected reward above the threshold $\theta$. The goal of the learner is to identify the correct outlier set $O^{*}$ from $\mathcal O$ by sampling the arms in the following sequential manner. 

Initially, the learner is presented with $n$ arms, he is therefore aware of all the probable combinations of the arms $\mathcal O$, but the reward distributions of the arms are unknown to the learner. At each round, the learner selects an arm to pull according to the sampling strategy, and the reward drawn from the corresponding distribution is revealed to the learner. This process continues until the learner has identified a set $O$ satisfying that $\mathbb P(O = O^*)\ge 1-\delta$ for a given confidence parameter $\delta$. The performance of the learner is measured by the sample complexity, which is the number of samples it requires with high probability.




\section{Algorithms and Results}
In this section, we illustrate an online learning algorithm for the outlier detection problem in the fixed confidence setting, and present the theoretical results including the correctness guarantee and sample complexity.


\subsection{The Algorithm}






The arms with expected reward deviating the mean $k$ standard deviation are referred to as the outliers. In order to efficiently identify the outliers, the algorithm needs to balance well between exploration for the rewards of the arms and exploration for the threshold. Two distinct sampling strategies are designed to address the \textit{double-exploration} dilemma. Specifically, we use a sequential sampling approach to estimate the expected rewards of the arms, and a random sampling approach to estimate the threshold separately.


\begin{itemize}
\item\textbf{Sequential Sample for The Arms}\; The learner pulls each arm $i$ in the candidate set once, and observes the rewards drawn from the corresponding distribution. The observed rewards are used to estimate the expected rewards of the arms.\\
\item\textbf{Random Sample for The Threshold}\; The learner samples an arm uniformly at random from the $n$ arms, and pulls the sampled arm twice. The observed rewards are used to estimate the threshold.
\end{itemize}

Let $m_{i,t}$ be the number of times arm $i$ is sampled prior to round $t$ in the process of sequential sample, and by $y_{i, 1}, y_{i, 2},\dots, y_{i, m_{i,t}}$ the sequence of the sampled rewards. The estimator for $y_i$ at round $t$ is the empirical mean of arm $i$ after $m_{i,t}$ samples, represented as 
\begin{align}
\hat y_{i,t} = \displaystyle\frac{1}{m_{i,t}} \sum_{l=1}^{m_{i,t}}y_{i, l}.\label{hat_y}
\end{align}

Let $m_{\theta,t}$ be the number of times the learner performs random sample for the threshold prior to round $t$, and by $(x_{1,1},x_{1,2}),(x_{2,1},x_{2,2}),\dots,(x_{m_{\theta,t},1},x_{m_{\theta,t},2})$ the sequence of associated rewards. The estimator for threshold $\theta$ at round $t$ is
\begin{align}
 \hat\theta_t = \hat\mu_{y,t} + k\hat\sigma_{y,t},\label{hat_theta}
\end{align}
where $\hat\mu_{y,t} = \displaystyle\frac{1}{m_{\theta,t}}\sum_{l=1}^{m_{\theta,t}} x_{l,1}$, and $\hat\sigma_{y,t} = \sqrt{\Bigg|\displaystyle\frac{1}{m_{\theta,t}}\sum_{l=1}^{m_{\theta,t}} x_{l,1}x_{l,2} - \displaystyle\frac{1}{{m_{\theta,t}}^2}\sum_{l=1}^{m_{\theta,t}}\sum_{h=1}^{m_{\theta,t}} x_{l,1} x_{h,2}\Bigg|}$.

\begin{figure*}[h]
   \centering 
   \subfigure[The initial round]{
   \includegraphics[width=.85\columnwidth]{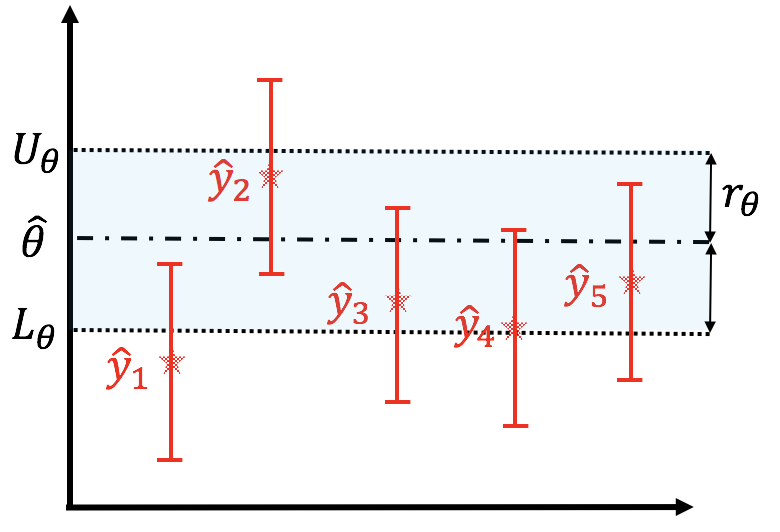}
   \label{fig:The number of samples w.r.t distinct number of samples}
   }
   \subfigure[The termination round]{
   \includegraphics[width=.85\columnwidth]{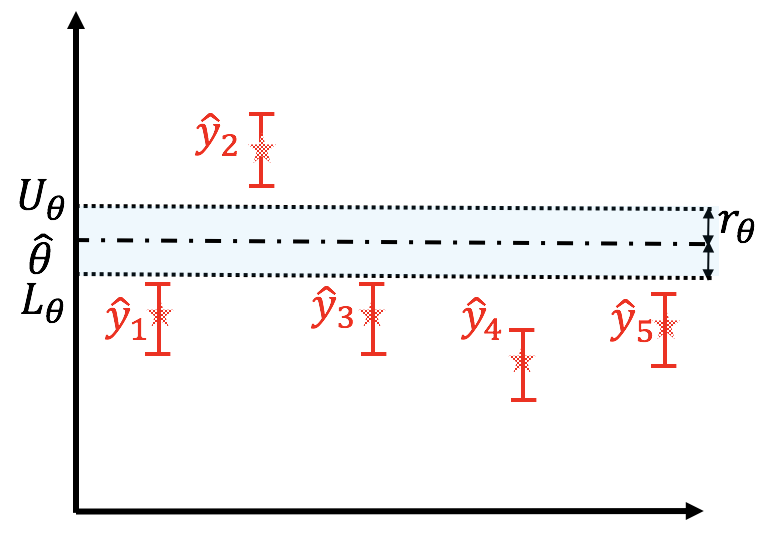}
   \label{fig:The number of samples w.r.t distinct number of samples}
   }
   \caption{Illustration of outlier detection based on confidence interval} 
   \label{fig:Algorithm Illustration}
\end{figure*}

The pseudo-code of our proposed algorithm \textup{Adaptive Double Exploration} (\textup{ADE}) is illustrated in Algorithm 1. Let $m_{a,t}$ be the number of times the learner performs sequential sampling prior to round $t$. In the initialization phase, we perform two sampling strategies separately, ensuring that $m_{\theta,t} \ge 1$ and $m_{a,t}\ge 1$ after the initialization steps.

At each round $t$, the algorithm chooses to sample for either the arms or the threshold, whichever has a larger confidence radius. As a result, the sampling for estimating the threshold and the rewards are balanced well in an adaptive manner. Then, it calculates the estimator $\hat y_{i,t}$ on the expected reward of each arm $i\in [n]$ and the estimator $\hat\theta_t$ on the threshold $\theta$, which are defined based on Eq.(\ref{hat_y}) and Eq.(\ref{hat_theta}). Besides, the confidence radius $r_{i,t}$ for each arm $i\in [n]$ and $r_{\theta,t}$ for the threshold $\theta$ are also maintained.

We separate the outliers from the normal arms based on the calculated confidence intervals in the following manner. If the lower confidence bound of arm $i$ at round $t$ (represented as $L_{i,t}$) is greater than or equal to the upper confidence bound of $\theta$ (represented as $U_{\theta,t}$), then arm $i$ is identified as an outlier. If $L_{\theta,t} \ge U_{i,t}$, then arm $i$ is regarded as a normal arm. Otherwise, arm $i$ remains in the candidate set and the algorithm continues to sample this arm in the next round. The arms that have been identified could be directly removed in the process of sequential sampling. As a result, the arms that are farther from the threshold could be pulled fewer times in the process of the sequential sampling. The algorithm is allowed to terminate until all the arms are determined as either an outlier or a normal arm. Figure 1 illustrates the confidence intervals of the arms and the threshold at the initial and termination rounds of the algorithm.







\begin{algorithm}[h]
\caption{Adaptive Double Exploration (ADE)}
\label{Framework:A basic algorithm for outlier detection 2}
\begin{algorithmic}[1]
\STATE{\textbf{Input:} $n$: the number of arms; $k$: threshold parameter}
\STATE{\textbf{Output:} the set of outliers}
\STATE{// Initialization}\label{Initialization_begin}
\STATE{$ S_0 \leftarrow [n]$, $m_{\theta,1}\leftarrow 0$, $m_{a,1}\leftarrow 0$, $t\leftarrow 1$}
\STATE{Sample an arm uniformly at random from the set $ S_0$, and pull this arm twice independently}
\STATE{$m_{\theta,t+1}\leftarrow m_{\theta,t}+1, m_{a,t+1}\leftarrow m_{a,t}, t\leftarrow t+1$}
\STATE{Pull each arm $i\in S_{0}$ once}
\STATE{$m_{\theta,t+1}\leftarrow m_{\theta,t}, m_{a,t+1}\leftarrow m_{a,t}+1$, $t\leftarrow t+1$}
\STATE{$S_{t} \leftarrow [n]$}
\STATE{Update $\hat\theta_{t}, r_{a,t}$, $r_{\theta,t}$ and $\hat y_{i,t}$ for $i\in S_{t}$}
\WHILE{$S_{t}\neq\emptyset$}


\STATE{// Sample for either the threshold or the arm whichever has larger radius}
\IF{$r_{a,t}\le r_{\theta,t}$}\label{Alg_Condi1}
\STATE{Sample an arm uniformly at random from the set $S_0$, and pull this arm twice independently}
\STATE{$m_{\theta,t+1}\leftarrow m_{\theta,t}+1, m_{a,t+1}\leftarrow m_{a,t}$}
\ELSE
\STATE{Sample each arm $i\in S_{t}$ and pull this arm once}
\STATE{$m_{\theta,t+1}\leftarrow m_{\theta,t}, m_{a,t+1}\leftarrow m_{a,t}+1$}
\ENDIF
\STATE{$t\leftarrow t+1$}
\STATE{Update $\hat\theta_{t}, r_{a,t}$, $r_{\theta,t}$ and $\hat y_{i,t}$ for $i\in S_{t-1}$}
\STATE{// Distinguish the outliers from the normal arms}
\FOR{$i\in S_{t-1}$}
\IF{$\hat y_{i,t} +r_{i,t} \le \hat\theta_{t} - r_{\theta,t}$}
\STATE{$N\leftarrow N \cup \{i\}$}
\STATE{$S_{t}\leftarrow S_{t-1} \setminus \{i\}$} \label{remove_determined_1}
\ELSIF{$\hat y_{i,t}-r_{i,t}\ge\hat\theta_{t}+r_{\theta,t}$}
\STATE{$O\leftarrow O \cup \{i\}$}
\STATE{$S_{t}\leftarrow S_{t-1} \setminus \{i\}$} \label{remove_determined_2}
\ENDIF
\ENDFOR
\ENDWHILE
\end{algorithmic}
\end{algorithm}

\textbf{Remark:} The samples obtained could be used to estimate both the threshold and the expected rewards of the arms. The sample complexity could be reduced by at most a half, which is still of the same order as the algorithm we presented. To ensure that the main idea of our algorithm could be presented in a concise and clear way, we decided to present this version of the algorithm.

One of the key issues of designing an efficient algorithm lies in the construction of the confidence intervals. Let $\delta_t = 3\delta/((n+4)\pi^2 t^2)$, the confidence intervals constructed for the threshold and the arms are illustrated in the following section.


\subsection{Construction of Confidence Interval}

\noindent Conditioned on $m_{\theta,t}=m$, the confidence interval of $\mu_y$ is illustrated in the following lemma.
\begin{lemm}\label{lem_mu}
With probability at least $1-2\delta_t$, we have $|\hat\mu_{y,t}-\mu_{y}|\le r_{\mu,t}$, where $r_{\mu,t} = R\sqrt{\displaystyle\frac{\log(1/\delta_t)}{2m}}$, and $\delta_t = \displaystyle\frac{3\delta}{(n+4)\pi^2 t^2}$.
\end{lemm}
\begin{lemm}\label{sigma_2_append}
With probability at least $1-\delta_{t}$, we have that 
\begin{align}
|\hat\sigma_{y,t}^2-\sigma_{y}^{2}|\le\epsilon_{\sigma, t},
\end{align}
where $\epsilon_{\sigma, t} = (R'+2bR)\sqrt{\displaystyle\frac{\log(6/\delta_{t})}{2m}}$.
\end{lemm}

\begin{proof}

\begin{align}
    &\mathbb E[\displaystyle\frac{1}{m}\sum_{l=1}^{m} x_{l,1}x_{l,2}] = \displaystyle\frac{1}{m}\sum_{l=1}^{m}\mathbb E[x_{l,1}x_{l,2}]\nonumber\\
    &= \displaystyle\frac{1}{m}\sum_{l=1}^{m}\mathbb E[\mathbb E[x_{l,1}x_{l,2}|y]]=\displaystyle\frac{1}{n}\sum_{i=1}^n y_i^2.\nonumber
\end{align}

\noindent Let $V_{t} = \displaystyle\frac{1}{m}\sum_{l=1}^{m} x_{l,1}x_{l,2} - \displaystyle\frac{1}{m^2}\sum_{l=1}^{m}\sum_{h=1}^{m} x_{l,1} x_{h,2}$. 
\begin{align}
\mathbb P&(|\sigma_{y}^{2} - V_{t}|\ge\epsilon)\nonumber\\
&\le\mathbb P\left(|\displaystyle\frac{1}{m} \sum_{l=1}^{m} (\mathbb E[x_{l,1}x_{l,2}]-x_{l,1}x_{l,2})|\ge\epsilon_1\right)\nonumber\\
&+ \mathbb P\left(|\displaystyle\frac{1}{m^2}\sum_{l=1}^{m}x_{l,1}\sum_{l=1}^{m}x_{l,2}-\mu_y^2|\ge\epsilon_2\right),\label{sigma_y}
\end{align}
\noindent where $\epsilon = \epsilon_1 + \epsilon_2$.

\noindent If $|(\displaystyle\frac{1}{m}\sum_{l=1}^{m}x_{l,1})^{2}-\mu_y^2|\le \epsilon_2$ and $|(\displaystyle\frac{1}{m}\sum_{l=1}^{m}x_{t,2})^{2}-\mu_y^2|\le \epsilon_2$, then we have $|\displaystyle\frac{1}{m^2}\sum_{l=1}^{m}x_{l,1}\sum_{l=1}^{m}x_{l,2}-\mu_y^2|\le \epsilon_2$, since $\displaystyle\frac{1}{m}\sum_{l=1}^{m}x_{l,1}\ge a \ge 0$, and $\displaystyle\frac{1}{m}\sum_{l=1}^{m}x_{l,2}\ge a \ge 0$. Therefore, the event $|\displaystyle\frac{1}{m^2}\sum_{l=1}^{m}x_{l,1}\sum_{l=1}^{m}x_{l,2}-\mu_y^2|\ge \epsilon_{2}$ implies that either $|(\displaystyle\frac{1}{m}\sum_{l=1}^{m}x_{l,1})^{2}-\mu_y^2|\ge \epsilon_2$ or $|(\displaystyle\frac{1}{m}\sum_{l=1}^{m}x_{l,2})^{2}-\mu_y^2|\ge \epsilon_2$. 
Thus, we have that
\begin{align}
&\mathbb P\left(|\displaystyle\frac{1}{m^2}\sum_{l=1}^{m}x_{l,1}\sum_{l=1}^{m}x_{l,2}-\mu_y^2|\ge\epsilon_{2}\right)\nonumber\\
&\le\mathbb P\left(|(\displaystyle\frac{1}{m}\sum_{l=1}^{m}x_{l,1})^{2}-\mu_y^2|\ge \epsilon_2\right) \nonumber\\
&+ \mathbb P\left(|(\displaystyle\frac{1}{m}\sum_{l=1}^{m}x_{l,2})^{2}-\mu_y^2|\ge \epsilon_2\right).\label{mu_y}
\end{align}

\noindent Combine Ineq.(\ref{sigma_y}) and Ineq.(\ref{mu_y}), we have
\begin{align}
\mathbb P&\left(|\sigma_{y}^{2} - V_{t}|\ge\epsilon\right)\nonumber\\
&\le\mathbb P\left(|\displaystyle\frac{1}{m} \sum_{l=1}^{m} (\mathbb E[x_{l,1}x_{l,2}]-x_{l,1}x_{l,2})|\ge\epsilon_1\right)\nonumber\\
&+\mathbb P\left(|(\displaystyle\frac{1}{m}\sum_{l=1}^{m}x_{l,1})^{2}-\mu_y^2|\ge\epsilon_2\right)\nonumber\\
&+ \mathbb P\left(|(\displaystyle\frac{1}{m}\sum_{l=1}^{m}x_{l,2})^{2}-\mu_y^2|\ge\epsilon_2\right).\nonumber
\end{align}
\noindent According to Hoeffding's inequality, we have $\mathbb P(|\displaystyle\frac{1}{m} \sum_{l=1}^{m} (\mathbb E[x_{l,1}x_{l,2}]-x_{l,1}x_{l,2})|\ge\epsilon_1)\le\delta_1$,
where $\epsilon_{1} = R' \sqrt{\log(2/\delta_1)/(2m)}$. Note that $\displaystyle\frac{1}{m} \sum_{l=1}^{m} x_{l,1} + \mathbb E[\displaystyle\frac{1}{m} \sum_{l=1}^{m} x_{l,1}]\le 2b$. Therefore,
\begin{align}
\mathbb P&\left(|(\displaystyle\frac{1}{m}\sum_{l=1}^{m}x_{l,1})^{2}-(\mathbb E[\displaystyle\frac{1}{m} \sum_{l=1}^{m} x_{l,1}])^2|\ge\epsilon_2\right)\nonumber\\
&\le\mathbb P\left(|\displaystyle\frac{1}{m}\sum_{l=1}^{m}x_{l,1} -\mathbb E[\displaystyle\frac{1}{m} \sum_{l=1}^{m} x_{l,1}]|\ge\displaystyle\frac{\epsilon_2}{2b}\right)\nonumber\\
&\le\delta_2,\nonumber
\end{align}
where $\epsilon_2 = 2bR\sqrt{\log(2/\delta_2)/(2m)}$.

\noindent Similarly, we have $\mathbb P\left(|(\displaystyle\frac{1}{m}\sum_{l=1}^{m}x_{l,2})^{2}-\mu_y^2|\ge\displaystyle\frac{\epsilon_2}{2b}\right)\le\delta_2$. Thus,
\begin{align}
\mathbb P&(|\sigma_{y}^{2} - V_{t}|\ge\epsilon)\nonumber\\
&\le\mathbb P\left(|\displaystyle\frac{1}{m} \sum_{l=1}^{m} \left(\mathbb E[x_{l,1}x_{l,2}]-x_{l,1}x_{l,2}\right)|\ge\epsilon_1\right)\nonumber\\
&+\mathbb P\left(|(\displaystyle\frac{1}{m}\sum_{l=1}^{m}x_{l,1})^{2}-\mu_y^2|\ge\epsilon_2\right)\nonumber\\
&+\mathbb P\left(|(\displaystyle\frac{1}{m}\sum_{l=1}^{m}x_{l,2})^{2}-\mu_y^2|\ge\epsilon_2\right)\nonumber\\
&\le\delta_1 + 2\delta_2.\nonumber
\end{align}


\noindent If we chose $\delta_1 = \delta_2 = \delta_t/3$, then $\epsilon_{1} = R' \sqrt{\displaystyle\frac{\log(6/\delta_{t})}{2m}}$ and $\epsilon_{2} = 2bR\sqrt{\displaystyle\frac{\log(6/\delta_{t})}{2m}}$. Thus, $|\sigma_{y}^{2} - V_t|\le\epsilon_{\sigma, t}$ with probability at least $1-\delta_{t}$, where $\epsilon_{\sigma, t} = \epsilon_1 + \epsilon_2 = (R'+2bR)\sqrt{\displaystyle\frac{\log(6/\delta_{t})}{2m}}$.

\noindent Note that $\hat\sigma_{y,t} = \sqrt{|V_{t}|}$, then with probability at least $1-\delta_t$,
\begin{align}
&|\sigma_y^2-\hat\sigma_{y,t}^2|\le|\sigma_{y}^{2} -V_{t}|\le\epsilon_{\sigma, t}.\nonumber
\end{align}

\end{proof}

\noindent Lemma 2 only informs us the gap between $\sigma_y^2$ and $\hat\sigma_{y,t}^2$, while what we are really interested is the gap between $\sigma_y$ and its estimator $\hat\sigma_{y,t}$. The confidence radius of $\sigma_y$ is constructed adaptively using $\hat\sigma_{y,t}$, which is formally stated in the following lemma.
\begin{lemm}
With probability at least $1-\delta_{t}$, we have $|\hat\sigma_{y,t} - \sigma_y|\le \sqrt{\displaystyle\frac{2}{U_{\sigma,t}}}\epsilon_{\sigma, t}$, where $U_{\sigma,t} = \min_{1\le\tau\le t}\{\hat\sigma_{y,\tau}^2+\epsilon_{\sigma,\tau}\}$, and $\epsilon_{\sigma,t} = (R'+2bR)\sqrt{\displaystyle\frac{\log(6/\delta_{t})}{2m}}$. 
\end{lemm}

\noindent Now we could derive the confidence bound of $\theta$ with the confidence bound of $\mu_{y}$ and $\sigma_{y}$.


\begin{lemm}\label{theta_confidence}

Conditioned on $m_{\theta, t} = m$, we have $|\hat\theta_{t} - \theta|\le r_{\theta,t}$ with probability at least $1-8\delta_{t}$, where $r_{\theta,t} = \left(R+\sqrt{2}k\displaystyle\frac{(R'+2bR)}{\sqrt{U_{\sigma,t}}}\right)\sqrt{\displaystyle\frac{\log(1/\delta_{t})}{2m}}$, $U_{\sigma,t} = \min_{1\le\tau\le t}\{\hat\sigma_{y,\tau}^2+\epsilon_{\sigma,\tau}\}$, and $\epsilon_{\sigma,t} = (R'+2bR)\sqrt{\displaystyle\frac{\log(1/\delta_{t})}{2 m}}$. 

\end{lemm}

\textbf{Remark:} The estimator for the threshold constructed by \citeauthor{zhuang2017identifying} \shortcite{zhuang2017identifying} is $\tilde\theta_{t} = \tilde\mu_{y,t} + k\sqrt{\displaystyle\frac{1}{n}\sum_{i=1}^n (\tilde y_{i,t} - \tilde\mu_{y,t})^2}$, where $\tilde y_{i,t} = \displaystyle\frac{1}{m_{i,t}}\sum_{j=1}^{m_{i,t}}x_i^{(j)}$, and $\tilde\mu_{y,t} = \displaystyle\frac{1}{n} \sum_{i=1}^n\tilde y_{i,t}$. They state that with probability at least $1-2\tilde\delta_{t}$, $|\tilde\theta_t - \theta|\le R\sqrt{\displaystyle\frac{l(k)}{2h(m)}\log\Big(\displaystyle\frac{1}{\tilde\delta_{t}}\Big)}$, where $l(k) = \left[\sqrt{\displaystyle\frac{(1+k\sqrt{n-1})^2}{n}}+\sqrt{\displaystyle\frac{k^2}{2\log(\pi^2n^3/(6\tilde\delta_t))}}\right]^2$, $\tilde\delta_t=\displaystyle\frac{6\delta}{\pi^{2}(n+1)t^2}$, and $h(m)$ is the harmonic mean of $m_{i,t}$ over all the arms. The confidence radius of $\tilde\theta$ increases with $n$, which might make the learning strategy inefficient especially when identifying outliers for large-scale dataset. Our approach alleviates this problem by constructing a confidence radius for $\theta$ which is independent of the number of arms (if we ignore logarithmic factors).






\begin{lemm}
Conditioned on $m_{i,t}=m$, we have $|y_i-\hat y_{i,t}|\le r_{i,t}$ with probability at least $1-2\delta_t$, where $r_{i,t} = R\sqrt{\log(1/\delta_t)/(2m)}$.
\end{lemm}


Applying the union bound over all arms and the possible number of iterations, we conclude that the confidence intervals hold for the threshold and any arm $i$ at any round $t$ with probability at least $1-\delta$, which is formally illustrated in the following lemma.
\begin{lemm}\label{A_Occur}
Define random event $\mathcal A = \{|y_i-\hat y_{i}|\le r_{i,t},|\theta-\hat\theta|\le r_{\theta,t}, \forall i, \forall m_a, \forall m_\theta\}$ ($t=m_a+m_\theta$), we have that event $\mathcal A$ occurs with probability at least $1-\delta$.
\end{lemm}
Therefore, with probability at least $1-\delta$, for any round $t$ and any arm $i\in [n]$, $L_{i,t} = \hat y_{i,t} - r_{i,t}$ and $L_{\theta,t} = \hat\theta_t - r_{\theta,t}$ could be regarded as the lower confidence bound of the expected reward $y_i$ and the threshold $\theta$ separately. Similarly, $U_{i,t} = \hat y_{i,t} + r_{i,t}$ and $U_{\theta,t} = \hat\theta_t + r_{\theta,t}$ are respectively the upper confidence bound of $y_i$ and $\theta$.

\subsection{Theoretical Results} 

\newtheorem*{thmm1}{Theorem 1 (\textnormal{Correctness})}

If $\mathcal A$ is satisfied, then at any round $t$, the arms contained in $N$ are normal arms, and the arms contained in $O$ are outliers. The algorithm terminates when all the arms have been assigned to either the normal arm set $N$ or the outlier set $O$. Therefore, the arm set returned by Algorithm 1 is the correct outlier set with probability at least $1-\delta$.

\begin{thmm1}
For any $\delta>0$, the algorithm returns the correct outlier set $O^*$ with probability at least $1-\delta$.
\end{thmm1}

For each arm $i\in [n]$, we define the gap between its expected reward and the threshold as $\Delta_i = |y_i - \theta|$, and let $\Delta_{\min} = \min_{i\in[n]} \Delta_i$ be the minimum gap among all the arms. The following theorem illustrates a problem-dependent sample complexity bound of Algorithm 1.

\newtheorem*{thmm2}{Theorem 2 (\textnormal{Sample Complexity})}

\begin{thmm2}
With probability at least $1-\delta$, the total number of samples of Algorithm 1 could be bounded by
\begin{align}
&O\Bigg(\sum_{i=1}^n \displaystyle\frac{1}{\Delta_i^2}\log\Big(\sqrt{\displaystyle\frac{n}{\delta}}\displaystyle\frac{1}{\Delta_i^2}\max\{1,(\displaystyle\frac{k}{\sigma_y})^2\}\Big) \nonumber\\
&+ \max\{1,(\displaystyle\frac{k}{\sigma_y})^2\}\displaystyle\frac{1}{\Delta_{\min}^2}\log\Big(\sqrt{\displaystyle\frac{n}{\delta}}\displaystyle\frac{1}{\Delta_{\min}^2}\max\{1,(\displaystyle\frac{k}{\sigma_y})^2\}\Big)\Bigg)\nonumber.
\end{align}
\end{thmm2}

\begin{figure*}[h]
\centering
\subfigure[$\#$ Samples vs. $\#$ Arms]{
\centering
\includegraphics[width = .85\columnwidth]{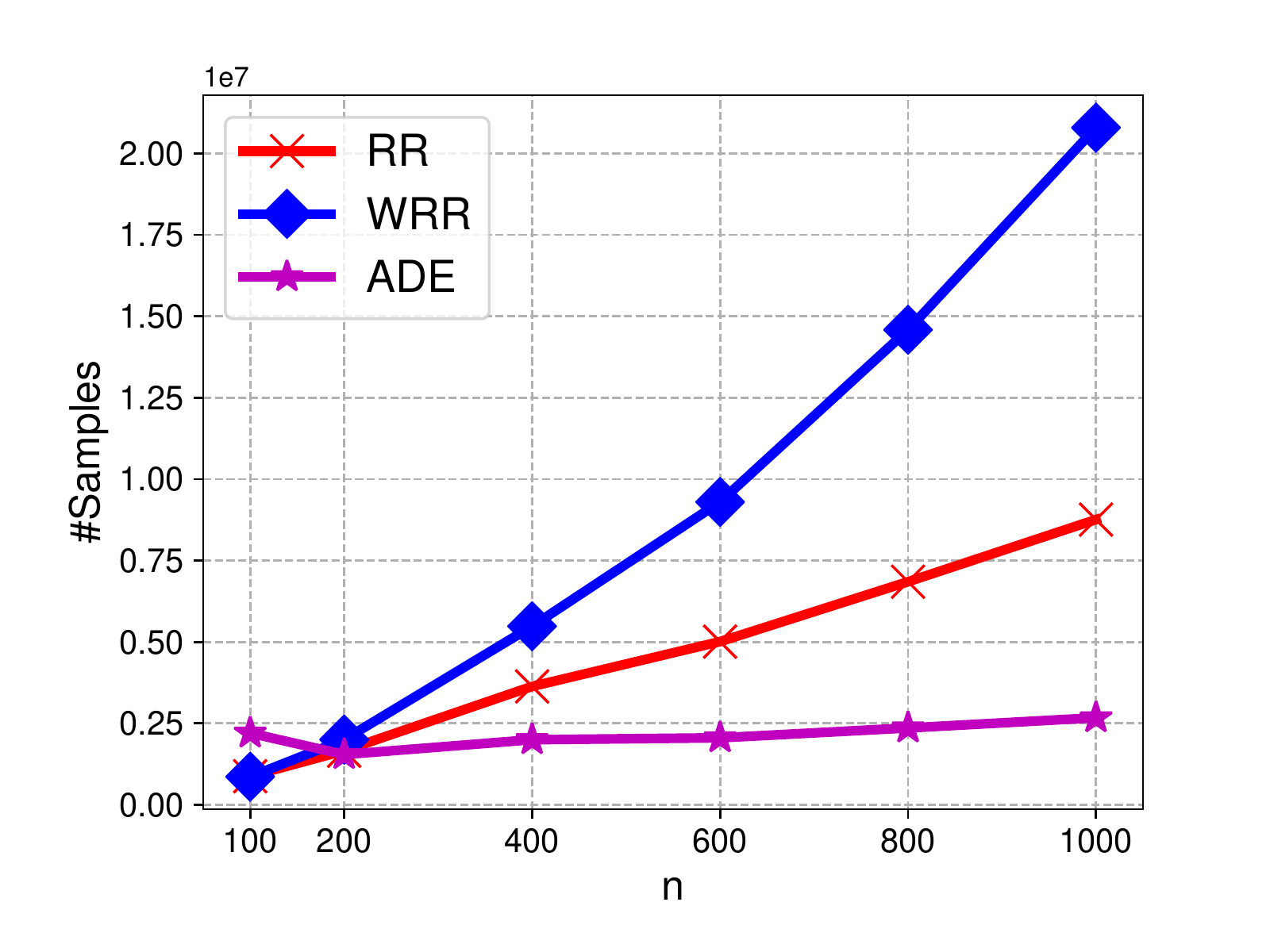}
\label{subfig2:The performance of distinct algorithms on Synthetic dataset}
}
\subfigure[$\#$ Samples vs. $\Delta_{\min}$]{
\centering
\includegraphics[width = .85\columnwidth]{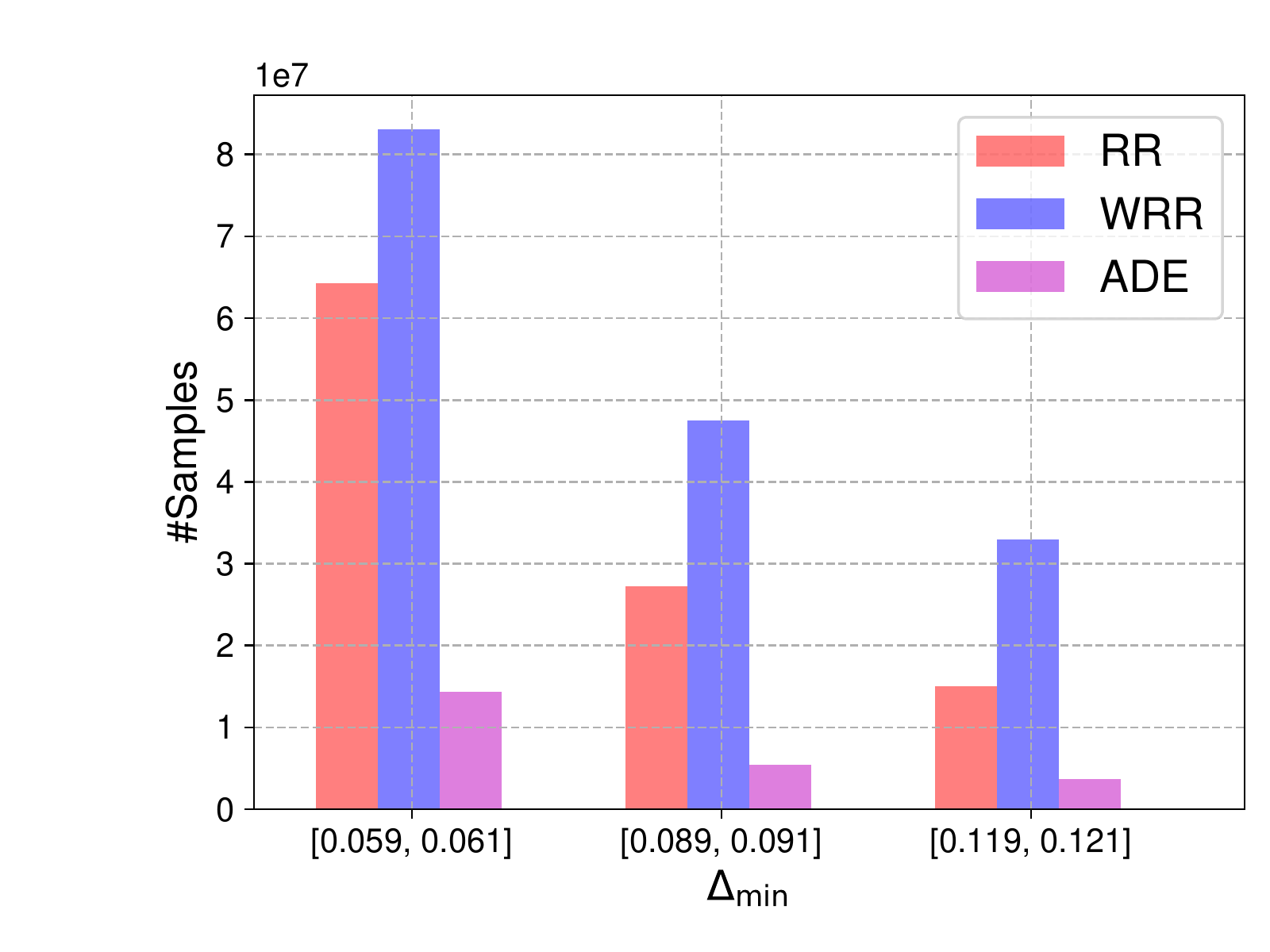}
\label{subfig1:The performance of distinct algorithms on Synthetic dataset}
}
\caption{The performance of distinct algorithms on synthetic datasets}
\label{fig:The performance of distinct algorithms on Synthetic dataset}
\end{figure*}

\paragraph{Comparison with RR and WRR:} 
\cite{zhuang2017identifying} proposed two algorithms named RR and WRR. RR samples each arm evenly in an iterative manner, while WRR allocates more samples to the unknown arms. The prior work used sequential samples to estimate the threshold, while we use random samples to estimate the threshold. Sequential sampling for the threshold results in an $h(m)$ factor in the denominator of the confidence radius, where $h(m)$ represents the harmonic mean. This implies that the confidence radius for the threshold scales with the number of arms, which further leads to the $n/\Delta_{\min}^2$ term in the upper bound of their sample complexity. Viewing from the threshold's definition, the random sampling approach provides a more natural way for constructing an estimator for the threshold. We perform random sampling for the threshold and construct an adaptive confidence radius which is independent of the number of arms (if we ignore logarithmic factors).

\section{Experiments}

\begin{figure*}[h]
\subfigure[The number of samples, $k=2$]{
\centering
\includegraphics[width = .67\columnwidth]{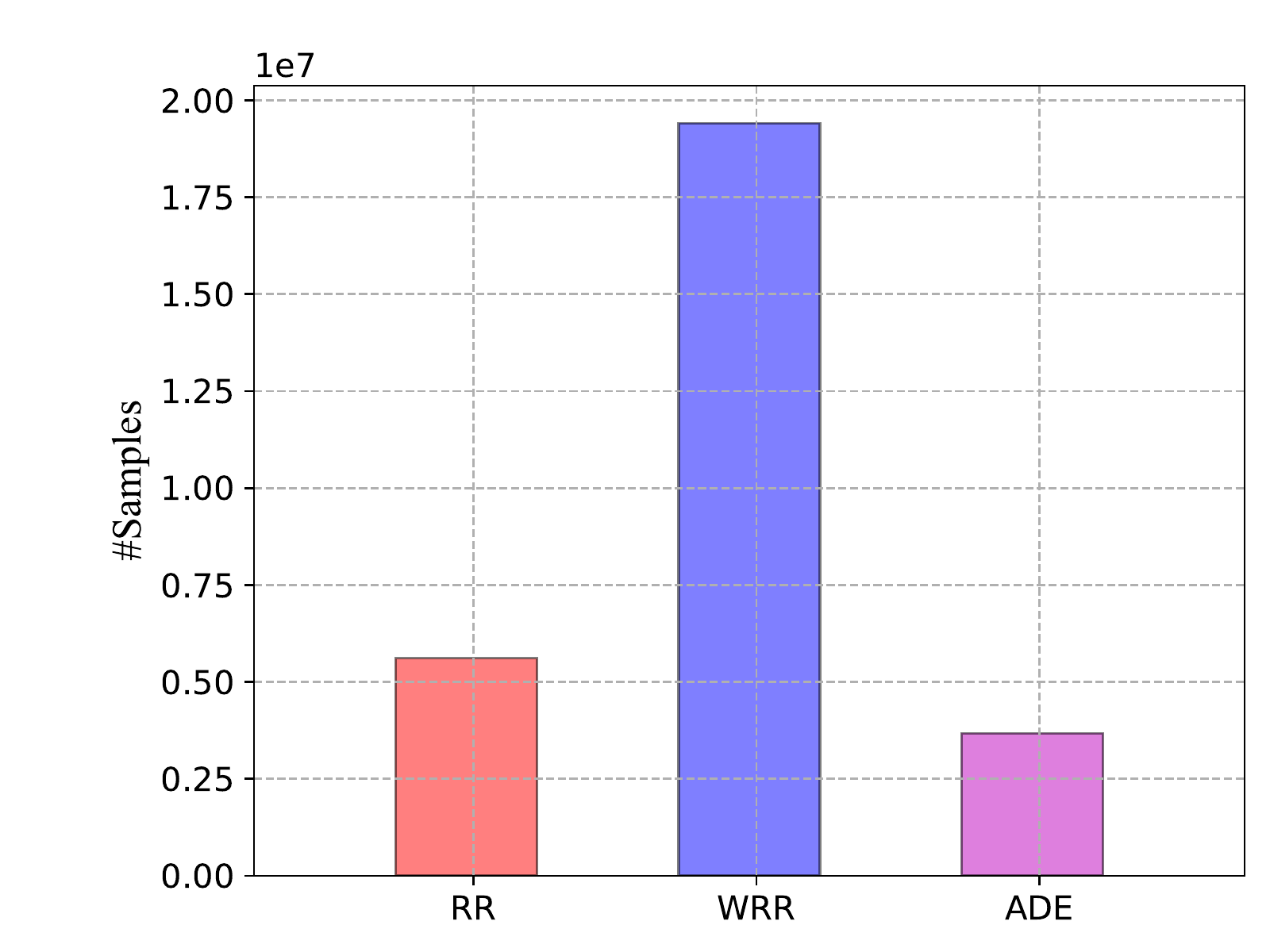}
\label{HyunCatherines a}
}
\subfigure[The number of samples, $k=3$]{
\centering
\includegraphics[width = .67\columnwidth]{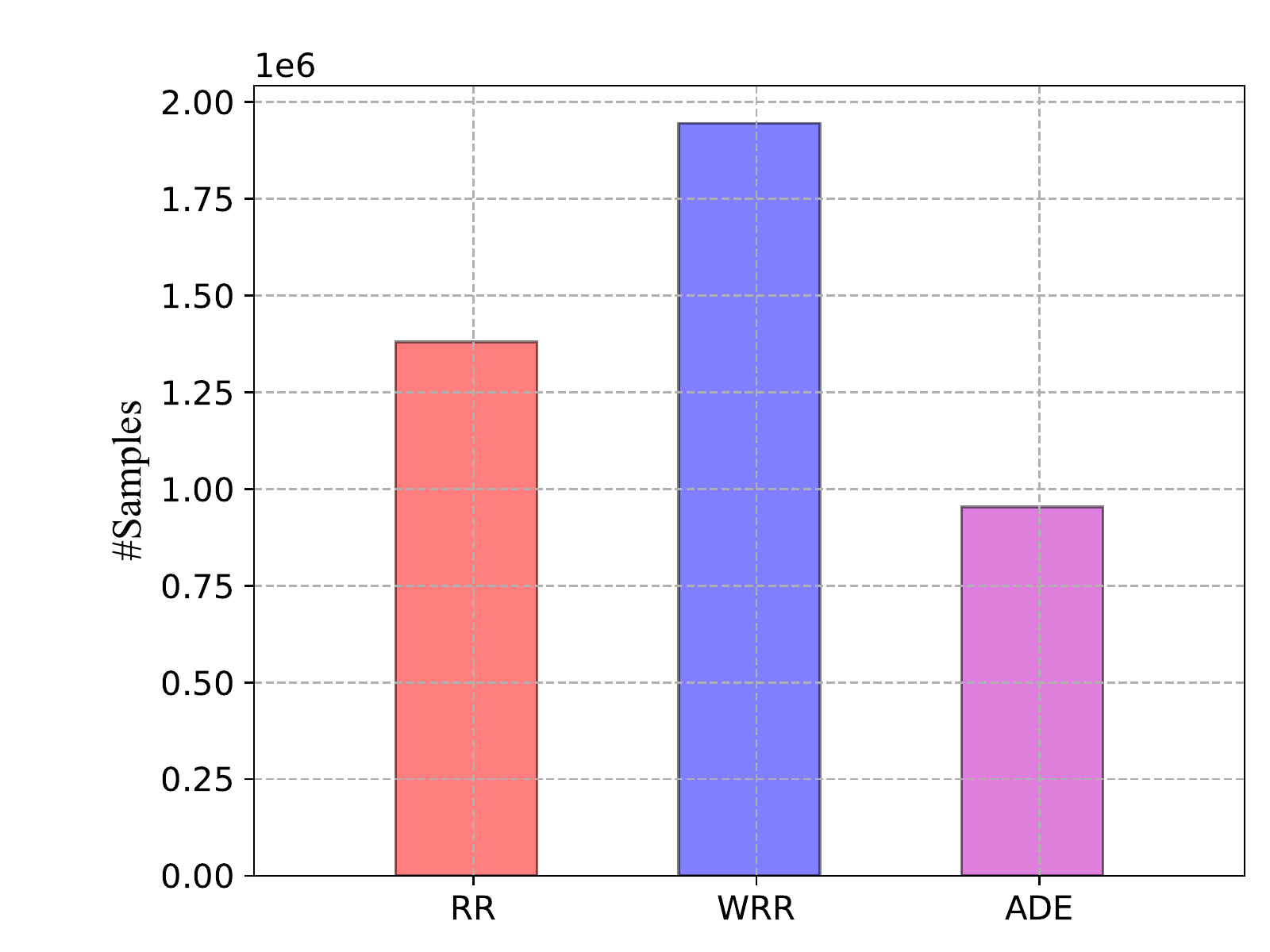}
\label{HyunCatherines b}
}
\centering
\subfigure[The distribution of $\Delta_i$]{
\centering
\includegraphics[width = .67\columnwidth]{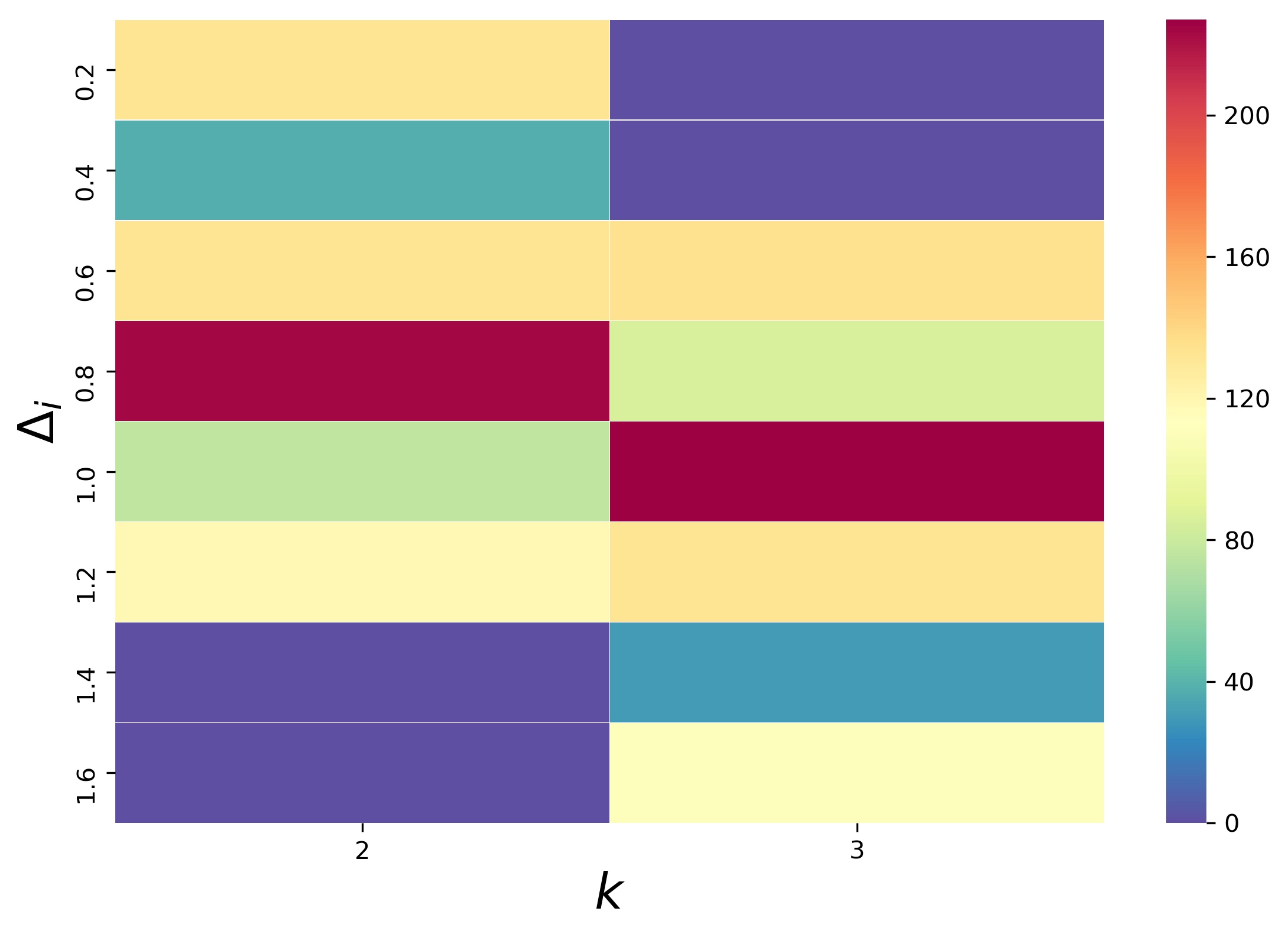} 
\label{HyunCatherines c}
}
\caption{The performance of distinct algorithms on HyunCatherines dataset}
\label{fig:The performance of distinct algorithms on HyunCatherines dataset}
\end{figure*}

In this section, we conduct experiments on both synthetic datasets and real datasets to show the performance of distinct algorithms. The confidence parameter $\delta$ is set as $0.1$ and the results are averaged across ten independent simulations. We compare the performance of our proposed \textup{ADE} algorithm with the state-of-the-art algorithms \textup{RR} and \textup{WRR} proposed by \citeauthor{zhuang2017identifying} \shortcite{zhuang2017identifying}.


\begin{itemize}

\item\textbf{Round-Robin (\textup{RR})}: \textup{RR} samples the arms in a round-robin way. Specifically, each arm is sampled in turn and in the following circular order, $1,2,\dots,n,1,2,\dots$. Thereby, each arm is sampled with identical frequency.

\item\textbf{Weighted Round-Robin (\textup{WRR})}: \textup{WRR} samples the arms in a weighted round-robin way. The arms that are not yet determined are sampled more frequently than the arms that have been determined. 

\end{itemize} 
Since \textup{RR} and \textup{WRR} might take a long period of time to terminate, we speed up these two algorithms by pulling an arm $1,000$ times at each round instead of pulling an arm once at each round. The termination conditions of all these algorithms are identical. Specifically, each algorithm terminates when there exists no overlapping between the confidence interval of the threshold and that of any arm.

\subsection{Experiments on the Synthetic Datasets}
We construct the synthetic datasets with distinct settings of $n$ and $\Delta_{\min}$. For each setting, we generate $10$ test cases independently. The expected reward of each arm is generated uniformly at random in $[0,1]$. The sampled reward corresponding to each pull is generated independently from the Bernoulli distribution. Accordingly, $R$ is set as $1$.

We compare the average number of samples used by distinct algorithms with regard to a distinct number of arms. Specifically, we vary $n$ as $\{100, 200, 400, 600, 800, 1000\}$, $k$ is set as $2.5$, the range of $\Delta_{\min}$ is $[0.1,0.2]$. We illustrate the experimental results in Figure \ref{subfig2:The performance of distinct algorithms on Synthetic dataset}, based on which we have the following observations. (1) When $n$ increases, the average number of samples used by each algorithm also tends to increase, which is consistent with the theoretical results. (2) When $n$ is as small as $100$, the number of samples used by \textup{RR} and \textup{WRR} are smaller than \textup{ADE}, and \textup{WRR} performs slightly better than \textup{RR}. The constant embedded in our algorithm counteracts its advantage when the hardness of the identification problem is small. (3) For all the remaining settings of $n$, \textup{ADE} has the best performance while \textup{WRR} performs the worst, and the improvement of \textup{ADE} over \textup{RR} and \textup{WRR} tends to be more significant with an increased number of arms. Figure \ref{subfig1:The performance of distinct algorithms on Synthetic dataset} shows the performance of distinct algorithms on synthetic datasets with various settings of $\Delta_{\min}$. We vary the range of $\Delta_{\min}$ as $\{[0.059,0.061], [0.089,0.091], [0.119,0.121]\}$, $n$ is fixed as $900$, $k$ is set as $2$. The average number of samples required by each algorithm tends to decrease when $\Delta_{\min}$ becomes larger, which is a key component that captures the hardness of the problem. \textup{ADE} performs better than both \textup{RR} and \textup{WRR} in these datasets.

\subsection{Experiments on the Real Datasets}

We characterize the performance of distinct algorithms on the real dataset HyunCatherines which is available at \url{http://ir.ischool.utexas.edu/square/data.html}. $722$ workers are extracted from the HyunCatherines dataset. The workers with extremely high error rates are regarded as outliers, and should be excluded from the crowd in order to obtain labels with high quality. Each worker could be regarded as an arm, and the error rate of this worker is the expected reward corresponding to this worker. Figure \ref{HyunCatherines a} and \ref{HyunCatherines b} illustrate the average number of samples each algorithm uses when $k$ is set as $2$ and $3$ respectively. It could be viewed that \textup{ADE} achieves considerable improvement over the state-of-the-art algorithms. Note that when $k$ is increased from $2$ to $3$, the number of samples required for these three algorithms all decreases. Although $k$ exists in the upper bound of the sample complexity of all these algorithms, the hardness of the problem does not necessarily increase with $k$. It is worth noting that the distribution of $\Delta_i$ may vary as $k$ changes. Figure \ref{HyunCatherines c} illustrates the distribution of $\Delta_i$ with distinct settings of $k$. For both datasets, the minimum gap between the arms and the threshold as well as the gaps of the bulk of arms tend to be larger when $k$ is increased from $2$ to $3$. This comparison explains the smaller usage of samples when $k$ is larger.

\section{Conclusion}
We studied the problem of outlier identification in the framework of the multi-armed bandit. The learner is asked to identify the arms with rewards above a threshold, which is a function of the rewards of all the arms. We proposed two distinct sampling strategies to address this \textit{double-exploration} dilemma, and constructed an adaptively adjusted confidence radius for the threshold which is independent of the number of arms. We put forward an algorithm that automatically balances between distinct sampling strategies. Theoretical analyses and experimental results illustrate the efficiency of our algorithm. An interesting direction for future work is to address this problem in the fixed budget setting. Another interesting question is whether it is possible to construct a better estimator for the threshold with a tighter confidence interval. Besides, it remains open to derive a lower bound for this problem.

 


\section*{Acknowledgements}
Xiaojin Zhang would like to thank Qiman Shao for helpful discussions.

\bibliographystyle{aaai}
\bibliography{references}


\begin{appendix}
\newtheorem{thm}{Theorem}[subsection]
\newtheorem{lem}[thm]{Lemma}
\newtheorem *{appenlem2}{Lemma 1}
\newtheorem *{appenlem3}{Lemma 2}
\newtheorem *{appenlem4}{Lemma 3}
\newtheorem *{appenlem5}{Lemma 4}
\newtheorem *{appenlem6}{Lemma 5}
\newtheorem *{appenlem7}{Lemma 6}
\newtheorem *{appenlem8}{Lemma 7}
\newtheorem{thm2}{Theorem}


\section{\leftline{Proof of Confidence Intervals}}

\noindent The confidence interval of $\mu_{y}$ could be directly constructed using Hoeffding's inequality, which is shown in the following proposition.

\begin{proposition}
Suppose $X_1, X_2, \dots, X_n$ are independent variables, and they are bounded in $[a,b]$. Then we have that,
\begin{align}
\mathbb P(|\displaystyle\frac{1}{n}\sum_{i=1}^n X_i - \mathbb E[\displaystyle\frac{1}{n}\sum_{i=1}^n X_i]|\ge \epsilon)\le 2\exp(-2n\epsilon^2/R^2),\nonumber
\end{align}
where $R = b - a$.
\end{proposition}



\noindent Conditioned on $m_{\theta,t}=m$, the confidence interval of $\mu_y$ is illustrated in the following lemma.
\begin{appenlem2}\label{lem_mu}
With probability at least $1-2\delta_t$, we have $|\hat\mu_{y,t}-\mu_{y}|\le r_{\mu,t}$, where $r_{\mu,t} = R\sqrt{\displaystyle\frac{\log(1/\delta_t)}{2m}}$, and $\delta_t = \displaystyle\frac{3\delta}{(n+4)\pi^2 t^2}$.
\end{appenlem2}
\begin{proof}
Firstly we show that $\hat\mu_{y,t}$ is an unbiased estimator of $\mu_y$. That is, $\mathbb E[\hat\mu_{y,t}] = \mu_{y}$. Since $x_{l,1}$ is the sampled reward of an arm which is uniformly sampled from the arm set $\{1,2,\dots,n\}$, we have
\begin{align}
&\mathbb E[x_{l,1}] = \mathbb E[\mathbb E[x_{l,1}|y]] = \displaystyle\frac{1}{n}\sum_{i=1}^{n} y_{i} = \mu_{y}.
\end{align}

\noindent Therefore, 
\begin{align}
&\mathbb E[\hat\mu_{y,t}] = \mathbb E[\displaystyle\frac{1}{m}\sum_{l=1}^{m} x_{l,1}]  = \displaystyle\frac{1}{m} m\mu_{y} = \mu_{y}.
\end{align}

\noindent The estimator of $\mu_y$ at round $t$ is constructed using $x_{l,1} (1\le l \le m)$, which are independent and bounded random variables. Therefore, the confidence interval for $\mu_y$ could be directly constructed using Hoeffding's inequality. Specifically, we have that
\begin{align}
\mathbb P(&|\hat\mu_{y,t}-\mu_{y}|\ge r_{\mu,t})\nonumber\\
&=\mathbb P(|\displaystyle\frac{1}{m}\sum_{l=1}^{m} x_{l,1}-\mathbb E[\displaystyle\frac{1}{m}\sum_{l=1}^{m} x_{l,1}]|\ge r_{\mu,t})\nonumber\\
&\le 2\exp(\displaystyle\frac{-2m r_{\mu,t}^{2}}{R^{2}}).\nonumber\\
\end{align}
Since $r_{\mu,t} = \sqrt{\displaystyle\frac{R^{2}\log(1/\delta_t)}{2m}}$, we have
\begin{align}
&\mathbb P(|\hat\mu_{y,t}-\mu_{y}|\ge r_{\mu, t})\le 2\delta_t.
\end{align}
\end{proof}

\begin{appenlem4}
With probability at least $1-\delta_{t}$, we have $|\hat\sigma_{y,t} - \sigma_y|\le \sqrt{\displaystyle\frac{2}{U_{\sigma,t}}}\epsilon_{\sigma, t}$, where $U_{\sigma,t} = \min_{1\le\tau\le t}\{\hat\sigma_{y,\tau}^2+\epsilon_{\sigma,\tau}\}$, and $\epsilon_{\sigma,t} = (R'+2bR)\sqrt{\displaystyle\frac{\log(6/\delta_{t})}{2m}}$. 
\end{appenlem4}
\begin{proof}
From Lemma 2, we have that
\begin{align}
|\hat\sigma_{y,t}^2-\sigma_{y}^{2}|\le\epsilon_{\sigma, t}.
\end{align}


\noindent If $\hat\sigma_{y,t}^2\ge\epsilon_{\sigma,t}$, then
\begin{align}
&|\hat\sigma_{y,t} -\sigma_y|\le\displaystyle\frac{\epsilon_{\sigma, t}}{\sqrt{\hat\sigma_{y,t}^2-\epsilon_{\sigma, t}}+\sqrt{\hat\sigma_{y,t}^2}}\label{case1}.
\end{align}

\noindent If $\epsilon_{\sigma,t}/3\le\hat\sigma_{y,t}^2\le\epsilon_{\sigma, t}$, then
\begin{align}
&|\hat\sigma_{y,t} - \sigma_y|\le \hat\sigma_{y,t}\label{case2}.
\end{align}

\noindent If $\hat\sigma_{y,t}^2\le\epsilon_{\sigma, t}/3$, then
\begin{align}
&|\hat\sigma_{y,t} - \sigma_y|\le \displaystyle\frac{\epsilon_{\sigma, t}}{\sqrt{\hat\sigma_{y,t}^2+\epsilon_{\sigma, t}}+\sqrt{\hat\sigma_{y,t}^2}}\label{case3}.
\end{align}

\noindent Ineq.(\ref{case1}) is due to $\sigma_{y}\ge \sqrt{\hat\sigma_{y,t}^2-\epsilon_{\sigma,t}}$ when $\hat\sigma_{y,t}^2\ge\epsilon_{\sigma,t}$, Ineq.(\ref{case2}) and Ineq.(\ref{case3}) are due to $|\hat\sigma_{y,t} - \sigma_y|\le \max\{\hat\sigma_{y,t}, \displaystyle\frac{\epsilon_{\sigma, t}}{\sqrt{\hat\sigma_{y,t}^2+\epsilon_{\sigma, t}}+\sqrt{\hat\sigma_{y,t}^2}}\}$, and $\hat\sigma_{y,t}\ge\displaystyle\frac{\epsilon_{\sigma, t}}{\sqrt{\hat\sigma_{y,t}^2+\epsilon_{\sigma, t}}+\sqrt{\hat\sigma_{y,t}^2}}$ when $\hat\sigma_{y,t}^2\ge\epsilon_{\sigma,t}/3$, $\hat\sigma_{y,t}\le\displaystyle\frac{\epsilon_{\sigma, t}}{\sqrt{\hat\sigma_{y,t}^2+\epsilon_{\sigma, t}}+\sqrt{\hat\sigma_{y,t}^2}}$ when $\hat\sigma_{y,t}^2\le\epsilon_{\sigma,t}/3$.

\noindent Let $U_{\sigma,t} = \min_{1\le\tau\le t}\{\hat\sigma_{y,\tau}^2+\epsilon_{\sigma,\tau}\}$. We have that



\begin{align}
|\hat\sigma_{y,t} - \sigma_y|&\le \displaystyle\frac{\sqrt{2}\epsilon_{\sigma, t}}{\sqrt{\hat\sigma_{y,t}^2+\epsilon_{\sigma, t}}}\nonumber\\
&\le\sqrt{\displaystyle\frac{2}{U_{\sigma,t}}}\epsilon_{\sigma, t}.
\end{align}
\end{proof}

\noindent Now we could derive the confidence interval of $\theta$ with the confidence interval of $\mu_{y}$ and $\sigma_{y}$.


\begin{appenlem5}
Conditioned on $m_{\theta, t} = m$, with probability at least $1-8\delta_{t}$, we have $|\hat\theta_{t} - \theta|\le r_{\theta,t}$, where $r_{\theta,t} = \Big(R+\sqrt{2}k\displaystyle\frac{(R'+2bR)}{\sqrt{U_{\sigma,t}}}\Big)\sqrt{\displaystyle\frac{\log(1/\delta_{t})}{2m}}$, $U_{\sigma,t} = \min_{1\le\tau\le t}\{\hat\sigma_{y,\tau}^2+\epsilon_{\sigma,\tau}\}$, and $\epsilon_{\sigma,t} = (R'+2bR)\sqrt{\displaystyle\frac{\log(1/\delta_{t})}{2 m}}$. 
\end{appenlem5}

\begin{proof}

Denote $r_{\mu,t}=R\sqrt{\displaystyle\frac{\log(1/\delta_{t})}{2m}}$, $r_{\sigma,t} = \sqrt{\displaystyle\frac{2}{U_{\sigma,t}}}\epsilon_{\sigma, t}$, where $U_{\sigma,t} = \min\{\hat\sigma_{y,t}^2+\epsilon_{\sigma,t}, U_{\sigma,t-1}\}$, and $\epsilon_{\sigma,t} = (R'+2bR)\sqrt{\displaystyle\frac{\log(1/\delta_{t})}{2 m}}$. 
From Lemma 1 and Lemma 3, we have
\begin{align}
\mathbb P(|\mu_{y}-\hat\mu_{y,t}|\ge r_{\mu,t})\le 2\delta_{t},
\end{align}
and
\begin{align}
\mathbb P(k|\sigma_{y}-\hat\sigma_{y,t}|\ge k r_{\sigma,t})\le 6\delta_{t}.
\end{align}

\noindent If we chose $r_{\theta,t} = r_{\mu,t}+k r_{\sigma,t}$. Then we have
\begin{align}
\mathbb P(|\hat\theta_t - \theta|\ge r_{\theta,t}) =& \mathbb P(|\hat\mu_{y,t}+k\hat\sigma_{y,t}-\mu_{y}-k\sigma_{y}|\ge r_{\theta,t})\nonumber\\
\le& \mathbb P(|\hat\mu_{y,t}-\mu_{y}|\ge r_{\mu,t})\nonumber\\
&+\mathbb P(|k\hat\sigma_{y,t} - k\sigma_{y}|\ge k r_{\sigma,t})\nonumber\\
\le& 8\delta_{t}.
\end{align}

\end{proof}

\begin{appenlem6}
Define random event $\mathcal A = \{|y_i-\hat y_{i}|\le r_{i,t},|\theta-\hat\theta|\le r_{\theta,t}, \forall i, \forall m_a, \forall m_\theta\}$ ($t=m_a+m_\theta$), then we have event $\mathcal A$ occurs with probability at least $1-\delta$.
\end{appenlem6}

\begin{proof}

At round $t$, conditioned on $m_{i,t}=m_a$, for any arm $i\in [n]$, by Hoeffding's inequality, we have,
\begin{align}
\mathbb P(|y_i-\hat y_{i}|\ge r_{i,t})\le 2\exp(-2r_{i,t}^2m_a/R^2) = 2\delta_t,\nonumber
\end{align}
where the equality is due to $r_{i,t} = R\sqrt{\displaystyle\frac{\log(1/\delta_t)}{2m_a}}$.

\noindent Combined with Lemma \ref{theta_confidence}, we have,
\begin{align}
&\mathbb P(\lnot\mathcal A)\nonumber\\
\le&\sum_{m_a=1}^{\infty}\sum_{m_{\theta}=1}^{\infty}\bigg[\sum_{i\in S_0}\mathbb P(|y_i-\hat y_{i}|\ge r_{i,(m_a+m_{\theta})})\nonumber\\&+\mathbb P(|\theta-\hat\theta|\ge r_{\theta,(m_a+m_{\theta})})\bigg]\nonumber\\
\le&\sum_{m_a=1}^{\infty}\sum_{i\in S_0}\mathbb P(|y_i-\hat y_{i}|\ge r_{i,m_a})+\sum_{m_{\theta}=1}^{\infty}\mathbb P(|\theta-\hat\theta|\ge r_{\theta,m_{\theta}})\nonumber\\
\le&\sum_{l=1}^{\infty}(2n+8)\delta_l\nonumber\\
=&\delta.
\end{align}

\noindent Thus for all iterations $t$, and all arms $i$, it is satisfied that
$|y_i-\hat y_{i,t}|\le r_{i,t}$ and $|\theta-\hat\theta_{t}|\le r_{\theta,t}$ with probability at least $1-\delta$.

\end{proof}

\section{Proof of Theorem 1}

\begin{thmm1}
For any $\delta>0$, the algorithm returns the correct outlier set $O^*$ with probability at least $1-\delta$.
\end{thmm1}

\begin{proof}
From Lemma 5, $\mathcal A$ is satisfied with probability at least $1-\delta$. If $\mathcal A$ is satisfied, then at any round $t$, for any arm $i\in O$ we have, $\theta\le\hat\theta_t+r_{\theta,t}\le\hat y_{i,t}-r_{i,t}\le y_i$. Thus, all arms contained in $O$ are outliers. Besides, at any round $t$, for any arm $i\in N$ we have, $y_i\le\hat y_{i,t}+r_{i,t}\le\hat\theta_t-r_{\theta,t}\le \theta$. Thus, all arms contained in $N$ are normal arms. Consequently, the arm set returned by Algorithm 1 is the correct outlier set with probability at least $1-\delta$.
\end{proof}

\section{Proof of Theorem 2}

\begin{appenlem7}\label{Delta_Relation}
If arm $i$ is determined as a normal arm or an outlier at round $t$, then for any $t'\le t$, we have that $\Delta_{i} < 2(r_{i,t'} + r_{\theta,t'})$ with probability at least $1-\delta$. 
\end{appenlem7}

\begin{proof}
\noindent Suppose for contradiction $\Delta_{i} \ge 2(r_{i,t'} + r_{\theta,t'})$. If arm $i$ is a normal arm and is identified at round $t$. Then we could infer that $t'$ is the round when arm $i$ is not yet confidently identified. 

\noindent Conditioned on $\mathcal A$, we have
\begin{align}
\hat\theta_{t'} -\hat y_{i,t'} + r_{i,t'} + r_{\theta,t'}\ge \nonumber\\
\theta -y_i\ge 2(r_{i,t'}+r_{\theta,t'}).
\end{align}
Then, 
\begin{align}
\hat\theta_{t'} - r_{\theta,t'} \ge \hat y_{i,t'} + r_{i,t'}.
\end{align}

\noindent The above equation implies that arm $i$ have been added to the normal set $N$ at round $t'$, which leads to a contradiction. Therefore, $\Delta_{i} < 2(r_{i,t'} + r_{\theta,t'})$.

\noindent Similarly, if $i$ is an outlier, conditioned on $\mathcal A$, we have $\Delta_{i} < 2(r_{i,t'} + r_{\theta,t'})$.

\noindent From Lemma 5, we know that event $\mathcal A$ occurs with probability at least $1-\delta$. Therefore, we have that $\Delta_{i} < 2(r_{i,t'} + r_{\theta,t'})$ with probability at least $1-\delta$. 

\end{proof}

\begin{thmm2} 
With probability at least $1-\delta$, the total number of samples of Algorithm 1 could be bounded by
\begin{align}
&O\Bigg(\sum_{i=1}^n \displaystyle\frac{1}{\Delta_i^2}\log\Big(\sqrt{\displaystyle\frac{n}{\delta}}\displaystyle\frac{1}{\Delta_i^2}\max\{1,(\displaystyle\frac{k}{\sigma_y})^2\}\Big) \nonumber\\
&+ \max\{1,(\displaystyle\frac{k}{\sigma_y})^2\}\displaystyle\frac{1}{\Delta_{\min}^2}\log\Big(\sqrt{\displaystyle\frac{n}{\delta}}\displaystyle\frac{1}{\Delta_{\min}^2}\max\{1,(\displaystyle\frac{k}{\sigma_y})^2\}\Big)\Bigg)\nonumber.
\end{align}
\end{thmm2}
\begin{proof}
Recalling that $r_{\theta, t} = (R+\sqrt{2}k(R'+2bR)/\sqrt{U_{\sigma,t}})\sqrt{\log(1/\delta_{t})/(2m_{\theta,t})}$, and that $r_{i, t} = R\sqrt{\log(1/\delta_{t})/(2m_{i,t})}$, where $\delta_{t} = 3\delta/((n+4)\pi^2 {t}^2)$, $U_{\sigma,t} = \min_{1\le\tau\le t}\{\hat\sigma_{y,\tau}^2+\epsilon_{\sigma,\tau}\}$, and $\epsilon_{\sigma,t} = (R'+2bR)\sqrt{\log(1/\delta_{t})/(2m_{\theta, t})}$. 

\noindent Initially, $r_{\theta, t}>r_{i,t}$, therefore the algorithm samples for $\theta$. Suppose the algorithm switches to sample for the arms at round $\tilde t$. Then we have that
\begin{align}
    r_{\theta, \tilde t}\le r_{i,\tilde t},
\end{align}

\noindent and that
\begin{align}
    r_{\theta, \tilde t-1}\ge r_{i,\tilde t-1},
\end{align}
for any arm $i$ in the candidate set.


\noindent We have, 

\begin{align}
    r_{\theta, \tilde t} &= (R+\displaystyle\frac{\sqrt{2}(R'+2bR)k}{\sqrt{U_{\sigma,\tilde t}}})\sqrt{\displaystyle\frac{\log(1/\delta_{\tilde t})}{2m_{\theta, \tilde t}}}\\
    &\ge \displaystyle\frac{1}{2} (R+\displaystyle\frac{\sqrt{2}(R'+2bR)k}{\sqrt{U_{\sigma,\tilde t-1}}})\sqrt{\displaystyle\frac{\log(1/\delta_{\tilde t})}{2m_{\theta, \tilde t-1}}}\label{ineq1}\\
    &\ge \displaystyle\frac{1}{2}R\sqrt{\displaystyle\frac{\log(1/\delta_{\tilde t})}{2m_{i, \tilde t-1}}}\label{ineq2}\\
    &= \displaystyle\frac{1}{2}R\sqrt{\displaystyle\frac{\log(1/\delta_{\tilde t})}{2m_{i, \tilde t}}}\\
    &=\displaystyle\frac{1}{2} r_{i,\tilde t},
\end{align}

\noindent where Ineq.(\ref{ineq1}) is due to $m_{\theta,\tilde t}=m_{\theta,\tilde t-1}+1$, and Ineq.(\ref{ineq2}) is due to $r_{\theta, \tilde t-1}\ge r_{i,\tilde t-1}$.

\noindent Consequently, for any round $t$, we have
\begin{align}
r_{\theta, t}\ge \displaystyle\frac{1}{2} r_{i,t}.
\end{align}

\noindent Combined with $U_{\sigma,t}\ge\sigma_y^2$, we have that

\begin{align}
     m_{\theta,t}  
\le 4\Big(1+\displaystyle\frac{\sqrt{2}(R'+2bR)k}{R\sigma_y}\Big)^2 m_{i, t}.\label{N_theta_i}
\end{align}

\noindent Let $t_i$ be the round when arm $i$ is added to either the normal arm set $N$ or the outlier arm set $O$, and $t_i'$ the last round when arm $i$ is sampled in the process of sequential sampling prior to round $t_i$. Then, according to our algorithm, we have $r_{\theta, t_i'}\le r_{i,t_i'}$. 
According to Lemma 6, we have that
\begin{align}
\Delta_i &\le 2(r_{i, t_i'} + r_{\theta, t_i'})\nonumber\\
&\le 4r_{i, t_i'}.
\end{align}

\noindent Thus, we have
\begin{align}
m_{i,t_i'}&\le\displaystyle\frac{8R^2}{\Delta_i^2} \log\Big(\displaystyle\frac{(n+4)\pi^2 t_i'^2}{3\delta}\Big)
\end{align}

\noindent From the definition of $t_i$ and $t_i'$, we also have that
\begin{align}
    m_{i,t_i+1} = m_{i,t_i'+1} = m_{i,t_i'}+1.
\end{align}

\noindent Therefore, 
\begin{align}
     m_{i,t_i+1}&\le\displaystyle\frac{8R^2}{\Delta_i^2} \log\Big(\displaystyle\frac{(n+4)\pi^2t_i'^2}{3\delta}\Big)+1\nonumber\\
     &\le\displaystyle\frac{8R^2}{\Delta_i^2} \log\Big(\displaystyle\frac{(n+4)\pi^2t_i^2}{3\delta}\Big)+1\label{N_i_t}
\end{align}

\noindent From Ineq.(\ref{N_theta_i}), we have that
\begin{align}
t_i &= m_{i,t_i+1} + m_{\theta, t_i+1}\nonumber\\
&\le\Big(1+4(1+\displaystyle\frac{\sqrt{2}(R'+2bR)k}{R\sigma_y})^2\Big)m_{i,t_i+1}\label{t}
\end{align}

\noindent Let $C = (n+4)\pi^2(1+4(1+\displaystyle\frac{\sqrt{2}(R'+2bR)k}{R\sigma_y})^2)^2$. Combine Ineq.(\ref{N_i_t}) with Ineq.(\ref{t}), we have
\begin{align}
&m_{i,t_i+1}\le 1+\displaystyle\frac{8R^2}{\Delta_i^2} \log(Cm_{i,t_i+1}^2/(3\delta)).
\end{align}

\noindent From Lemma 8 in \cite{antos2010active}, we have that
\begin{align}
&m_{i,t_i+1}\le\displaystyle\frac{32R^2}{\Delta_i^2}\log(\displaystyle\frac{16R^2}{\Delta_i^2})+\displaystyle\frac{16R^2}{\Delta_i^2}\log\Big(C/(3\delta)\Big)+2\nonumber\\
&=O\Big(\displaystyle\frac{1}{\Delta_i^2}\log(\sqrt{\displaystyle\frac{n}{\delta}}\displaystyle\frac{1}{\Delta_i^2}\max\{1,(k/\sigma_y)^2\})\Big).
\end{align}

\noindent Suppose arm $i^*$ is the last arm that is identified by the algorithm, and it is identified at round $t_{i^*}$. From Ineq.(\ref{N_theta_i}), we have that

\begin{align}
&m_{\theta, t_{i^*}+1}\le 4(1+\displaystyle\frac{\sqrt{2}(R'+2bR)k}{R\sigma_y})^2m_{i^*,t_{i^*}+1}\nonumber\\
&= O(\max\{1,(\displaystyle\frac{k}{\sigma_y})^2\}\displaystyle\frac{1}{\Delta_{\min}^2}\log(\sqrt{\displaystyle\frac{n}{\delta}}\displaystyle\frac{1}{\Delta_{\min}^2}\max\{1,(\displaystyle\frac{k}{\sigma_y})^2\})).
\end{align}

\noindent Thus, the total number of samples of Algorithm 1 could be bounded by
\begin{align}
&O\Big(\sum_{i=1}^n \displaystyle\frac{1}{\Delta_i^2}\log\Big(\sqrt{\displaystyle\frac{n}{\delta}}\displaystyle\frac{1}{\Delta_i^2}\max\{1,(\displaystyle\frac{k}{\sigma_y})^2\}\Big) \nonumber\\
&+ \max\{1,(\displaystyle\frac{k}{\sigma_y})^2\}\displaystyle\frac{1}{\Delta_{\min}^2}\log\Big(\sqrt{\displaystyle\frac{n}{\delta}}\displaystyle\frac{1}{\Delta_{\min}^2}\max\{1,(\displaystyle\frac{k}{\sigma_y})^2\}\Big)\Big)\nonumber.
\end{align}
\end{proof}

\end{appendix}

\end{document}